\documentclass{article}

\usepackage{microtype}
\usepackage{graphicx}
\usepackage{subfigure}
\usepackage{booktabs} % for professional tables

\usepackage{hyperref}
\usepackage{enumitem}
\usepackage{caption}
\usepackage{newfloat}

\DeclareFloatingEnvironment[fileext=lop]{procedure}

\usepackage[color=green,size=tiny]{todonotes}

\usepackage{amsthm}
\usepackage{amsfonts}
\usepackage{amsmath}
\usepackage{amssymb}
\usepackage{bm}

\DeclareMathOperator*{\argmax}{arg\,max}
\DeclareMathOperator*{\argmin}{arg\,min}

\usepackage{lipsum}
\usepackage{lmodern}
\usepackage{tcolorbox}
\usepackage{cuted}
\newtheorem{theorem}{Theorem}

\usepackage[accepted]{icml2021}

\icmltitlerunning{Few-shot Language Coordination by Modeling Theory of Mind}

\begin{document}

\twocolumn[
\icmltitle{Few-shot Language Coordination by Modeling Theory of Mind}



\icmlsetsymbol{equal}{*}

\begin{icmlauthorlist}
\icmlauthor{Hao Zhu}{cmu}
\icmlauthor{Graham Neubig}{cmu}
\icmlauthor{Yonatan Bisk}{cmu}
\end{icmlauthorlist}

\icmlaffiliation{cmu}{Language Technologies Institute, Carnegie Mellon University}

\icmlcorrespondingauthor{Hao Zhu}{zhuhao@cmu.edu}

\icmlkeywords{Machine Learning, ICML}

\vskip 0.3in
]



\printAffiliationsAndNotice{}  % leave blank if no need to mention equal contribution

\begin{abstract}
% \gn{This is probably too sudden, you're assuming more or less that people already know the problem that you're working on. I'd add an extra sentence here first explaining what the problem is and why it's important.}
% \gn{``self-play'' is over-broad. It can also apply to playing Go or Chess for example. It'd be nice to be more focused ``communicative agents''} 
\emph{No man is an island.} Humans communicate with a large community by coordinating with different interlocutors within short conversations. This ability has been understudied by the research on building neural communicative agents. We study the task of few-shot \emph{language coordination}: agents quickly adapting to their conversational partners' language abilities. 
Different from current communicative agents trained with self-play, we require the lead agent to coordinate with a \emph{population} of agents with different linguistic abilities, quickly adapting to communicate with unseen agents in the population. This requires the ability to model the partner's beliefs, a vital component of human communication. %We instead build environments where agents are required to communicate with partners using different language distributions within a few interactions, where the models have the pressure to adapt their language efficiently.
Drawing inspiration from theory-of-mind (ToM; \citet{premack1978does}), we study the effect of the speaker explicitly modeling the listeners' mental states. The speakers, as shown in our experiments, acquire the ability to predict the reactions of their partner, which helps it generate instructions that concisely express its communicative goal.
We examine our hypothesis that the instructions generated with ToM modeling yield better communication performance in both a referential game and a language navigation task.
Positive results from our experiments hint at the importance of explicitly modeling communication as a socio-pragmatic progress. Code can be found at \url{https://github.com/CLAW-Lab/ToM}.
\end{abstract}

\section{Introduction}

Natural language is an ubiquitous communication medium between human interlocutors, and is shaped by the desire to efficiently cooperate and achieve communicative goals \cite{gibson2019efficiency}.
Because of this, there has been interest in creating artificial agents that mimic this communication process, with a wide variety of works examining communication between agents via either completely artificial \emph{emergent} language \cite{wagner2003progress, bouchacourt2018agents, li2019ease, kharitonov2020entropy}, or through \emph{natural} language such as English \cite{lazaridou2016multi,lowe2019interaction}.
In general, these methods model interaction between a pair of agents, a \emph{speaker} and a \emph{listener} that attempt to jointly achieve a goal where the language use is learned to optimize success with respect to a collaborative task (for example, in Fig. \ref{fig:demo} the speaker instructs an embodied agent to perform a task in an environment).

% Recent years have witnessed the success of goal-oriented communicational agents coordinating using language \cite{wagner2003progress, lazaridou2016multi, bouchacourt2018agents, li2019ease}.
% Through discrete channels, these agents could communicate to achieve perfect performance  in referential games as efficient as possible \cite{kharitonov2020entropy}.
% However, most of these agents are trained via self-play, so that a common language can emerge.
% To restrict the communication to natural language, some researchers resort to supervised learning with annotations, but without additional self-play, agents cannot achieve as high performance \cite{lowe2019interaction}.

\begin{figure}[!t]
    \includegraphics[width=\linewidth]{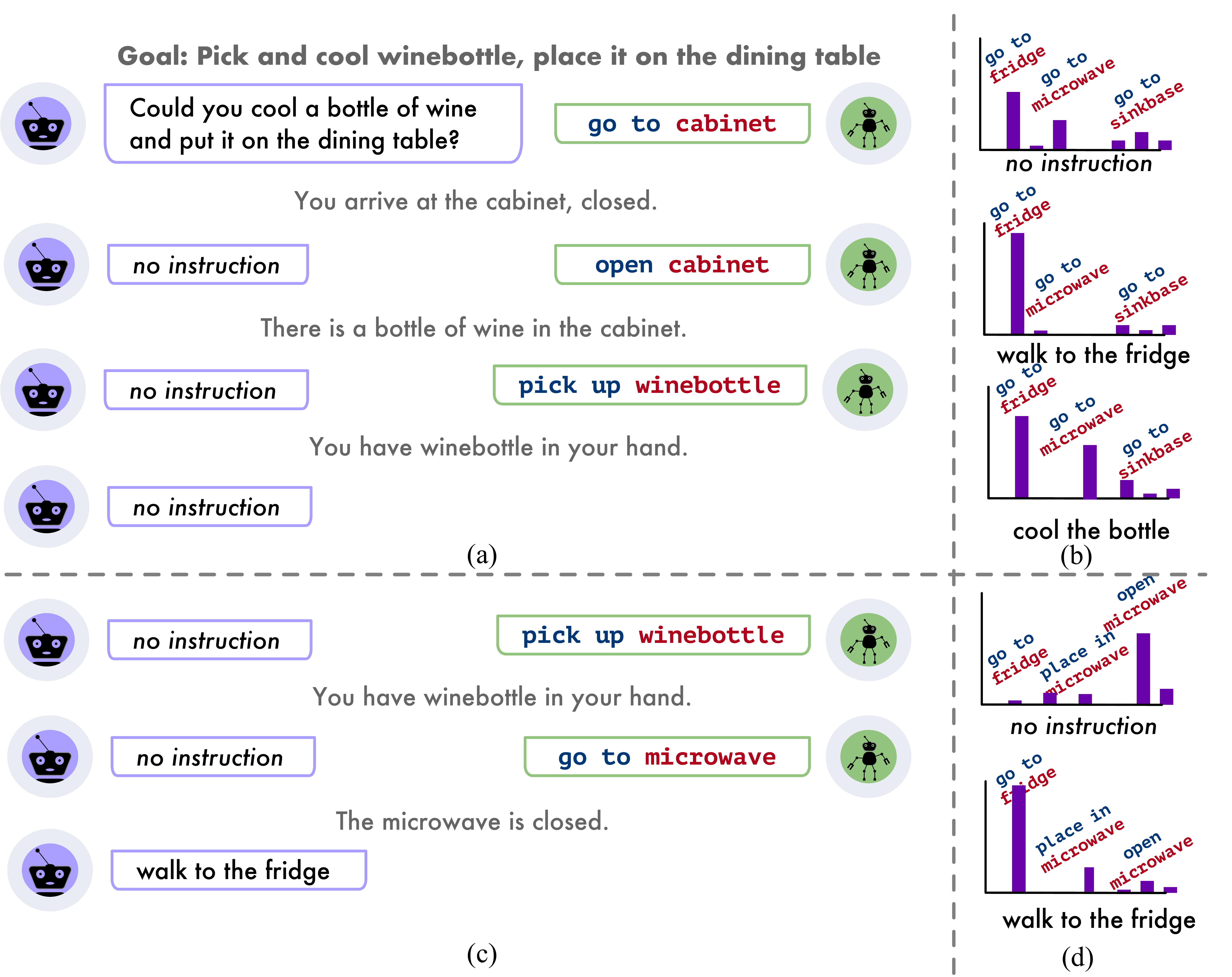}
    \caption{A conversation between a speaker and a listener collaboratively solving a navigation task. (a) At the start of the task, a goal (bold font) is given to the speaker (purple robot head). The speaker first gives a task-level instruction. Without previous knowledge of the listener, the speaker thinks the listener (green robot) will proceed to the fridge after three correct actions (monospace font) in a row. Grey observations are given by the environment after each action. (b) shows the belief of the speaker about the listener's action after a few instruction candidates. Note that to keep instructions concise the speaker chooses ``no instruction'' over ``walk to fridge'' despite the higher probability of listener taking correct action given the latter instruction.  (c) After the listener makes a mistake by going to the microwave, the speaker figures out that the listener cannot understand ``cool'' in the high-level instruction given, and gives low-level instruction ``walk to the fridge''. (d) shows the belief of speaker at this time step. Note that the probability of action ``\texttt{\color{blue}go to \color{red}fridge}'' without instruction decreases due to the wrong action of the listener.
    % \ybi{typo: windbottle} 
    %\gn{It might be nicer to have this on the second page.}
    }
    \label{fig:demo}
\end{figure}

However, in contrast to this setup, human speakers interact with not a single listener, but many different conversational partners.
In doing so, they also adapt to each other's language within short conversations. 
One of the representative phenomena is entrainment, in which interlocutors align their language on both acoustic-prosodic and lexical dimensions during communication \cite{brennan1996conceptual, levitan2018acoustic}. 
These issues remain mostly unattested in previous work on multi-agent communication --- it remains an open question how to train an agent which can adapt to communicating with novel partners quickly. 

We define this adaptive communication problem as \emph{few-shot language coordination}. As shown in Fig. \ref{fig:demo}, within a few rounds of the game, the speaker must adapt its language based on the responses from the listener to achieve the \emph{communicative goal} of helping the listener perform the correct actions as many times as possible. This few-shot coordination setting provides the agents with the pressure to adapt on-the-fly -- something that current models generally cannot achieve, even those that model pragmatics, e.g. the rational speech act model \cite{frank2012predicting}.

% \gn{Commented this out because I don't think it's so important, and it had a few problems as commented below.}
% We attribute this failure of current approaches to the fact that they do not model \emph{theory of mind} \gn{what is theory of mind? add a citation? it's not clear to me why the rational speech acts model isn't modeling theory of mind?}. 

%Human use language in a more robust way than trained agents. Even children have the ability to adapt to noisy linguistic inputs based on common sense and social-pragmatics cues \cite{yurovsky2017preschoolers} \gn{This is good, but not super-relevant to the story here}. Additionally, adults have the ability to adapt to languages of children \cite{snow1972mothers,bruner1985child}. Another example is language entrainment, the phenomenon in communications of the process of conversation partners converge on the same terms to eliminate ambiguity even if Grice's maxim of quantity. \cite{brennan1996conceptual} While people have made progress in understanding terms in goal-oriented games \cite{wang2016games}, adopting terms that are understandable to the listeners is still an open problem. 

Developmental psychology argues for the importance of Theory of Mind (ToM), the understanding of others’ mental states, and also the ability of interlocutors to act upon others mental states to achieve  desired effects \cite{tomasello2018children}. 
In this paper, we study the importance of modeling the beliefs of one's conversational partner to achieve few-shot language coordination.
In particular, we train a model that quickly adapts to predict the actions of the listeners in real time. At each time step, the speaker predicts the listener's likely next action given all possible instructions and the listener's previous actions.
This is essentially a few-shot learning problem, which we attack with model-agnostic meta-learning (MAML) \cite{finn2017model}.
In order to achieve the communicative goal, the speaker chooses the best instruction to give so that the probability of the listener performing the correct actions is maximised. 

We expect the resulting agent to not only mimic humans' ability to model a listener's mental state but also to leverage this estimate to choose better instructions. Through empirical evaluation, we aim to answer the question:

\begin{quote}
\it
    Can an agent, equipped with a model of theory of mind, quickly adapt to a listener's language in a few-shot language coordination game?
    %For communication agents, is it possible to predict the actions of listeners within a small number of interactions? Can a ToM model help adapt to listeners' language in few-shot language coordination? 
\end{quote}

Our experiments answer in the affirmative in the both a referential game and a vision-language navigation setting.

\section{Few-shot Language Coordination}
Consider again the example depicted in Fig.~\ref{fig:demo}.
A speaker model observes the goal of the current task (``pick and cool wine bottle, place it on the dining table''), and sends messages to a listener model, in an attempt to help it finish the task.
The listener model chooses actions to take given the latest and previous instructions.
If the listener makes a mistake, such as going to the microwave instead of the refrigerator, the speaker should realize that the listener misunderstood some aspect of the high-level instructions in the first message. This misstep informs the speaker's model of the listener's ability,  leading them to give lower-level instructions to help correct the trajectory. Through \emph{several games}, the speaker gathers enough data to provide customized messages for individual listeners.
This is only possible if a proper prior is provided to the speaker. The simplest prior can be hand-coded rules, e.g. if listener cannot understand abstract instructions, try simpler ones. 
However, to pursue a more general and powerful communication model (e.g. knowing when to simplify vs. rephrase),
%\yb{Hao, is this ok? attempting to resolve my commented out comment}
we study whether this kind of few-shot language coordination can be \emph{learned} by playing with a \emph{population} of listeners. This section proposes a method to construct meaningful populations. %\yb{Can you expand/change these last two sentences to make it clear what's being learned? It sort of seems like we're simply saying that the only thing you can/should learn is the hand-coded rule from the previous sentence?}\hao{It is a little bit hard to articulate, since the prior picked up by the speaker differs from population population. For an abstract population, it is hard to describe what should be learned other than a ``prior''.}

% How might we train the a model to solve this few-shot language coordination problem? This section proposes an automatic procedure for answering this question. 
\subsection{Asymmetric Speaker-Listener Games}

% \gn{Should this be explained before ``population''? It seems like that would be more natural: games -> population -> theory of mind.}

Following previous work on communicative agents \cite{lazaridou2016multi,cao2018emergent,lowe2019interaction}, we use goal-oriented language games 
%\gn{you shouldn't capitalize common nouns, even as part of a technical term.}
as the test bed for few-shot language coordination. A general goal-oriented language game provides an environment where the participants uses language to communicate with each other to achieve the given goal.
% \gn{``open communication'' and ``full access to the world'' are still somewhat vague. Also, this is not a complete sentence.}
We consider the most basic setting of a simplified two-player shared-goal multi-round setup:
\begin{description}[align=left]
    \item[Environment:] The environment is defined by $\mathcal{O}$bservation space, $\mathcal{A}$ction space, $\mathcal{G}$oal space and transition function $E:\mathcal{O}\times \mathcal{A}\rightarrow\mathcal{O}\times \mathcal{G}$. At the start of each game, the environment provides the speaker with a goal and both participants with observations after each action is taken by the listener. A new game starts after the previous one succeeds or reaches a maximum number of steps. 
    %\gn{This is defined mathematically in the algorithm, right? Be good to do the same here.} 
    \vspace{-2pt}
	\item[Participants:] The participants consist of a speaker and a listener sending and receiving natural language messages. After observing the goal, the speaker gives an instruction to the listener, and the listener performs an action in the environment. If the game is sequential, the speaker can also give an instruction after each action until the game is solved or the maximum number of steps is reached. The speaker is a message-and-action producing model defined by the \emph{vocabulary} $\Sigma$; the space of \emph{observations} $\mathcal{O}$; the space of \emph{actions} $\mathcal{A}$; and a \emph{model} $f:\mathcal{O}\times\mathcal{G}\rightarrow\Sigma^*\times \mathcal{A}$. The listener is an instruction-follower defined by the same vocabulary $\Sigma$, observation space $\mathcal{O}$, and space of \emph{actions} $\mathcal{A}$ as the speaker; and a \emph{model} $g: \Sigma^* \times \mathcal{O} \rightarrow \mathcal{A}$. 
	\item[Multi-round Games:] The pair of participants will play a \emph{session} of $N$ rounds of games, which are sampled independently. Different from single-round games ($N\!=\!1$) used in most previous work \cite{lazaridou2016multi, cao2018emergent, fried2018unified, lowe2019interaction}, the participants keep the memory of past games in the same session. Multi-round games are not only more general than %widely used 
	single-round games,
	%\gn{Since single-round games are used in previous work, it might be best to introduce them first (with citations) before explaining how you modify the setting.}
	but are essential to few-shot language coordination, because participants have the opportunity to adapt to the interlocutors by learning from feedback during previous rounds. 
	%\yb{I don't think this section is really clear about what is ``more general"?  Can we rewrite or expand this one?} \gn{Stating that single-round games are $N=1$ is probably sufficient.}
\end{description}

Note that the listeners in this setting cannot directly observe the goal, so the speakers need to use instructions to inform the listeners about the final or intermediate goals of each game. Within $N$ rounds, the speaker needs to adapt to the listener's knowledge to provide the most effective instructions.

\begin{figure*}[!t]
\begin{small}
\begin{tcolorbox}[left=2pt,right=2pt,top=2pt,bottom=2pt]
\textbf{Training Theory-of-Mind Model for Few-shot Language coordination} 

\vspace{5pt}

%Given
%\begin{itemize}[noitemsep,topsep=5pt,parsep=2pt,partopsep=10pt,leftmargin=30pt]
%	\item $n$ training listeners $l = {l_i}_{i=1}^n\in \mathcal{o}\times\mathcal{i}\rightarrow \mathcal{a}  \text{ }(i=1, 2, \dots, n)$ sampled from the distribution $\mathcal{d}_{\text{listener}}$
%	\item language game environment $e$: $\mathcal{o}\times\mathcal{a}\rightarrow \mathcal{o}$
%	\item speaker $s$: $\mathcal{o}\rightarrow\mathcal{i}^+$
%	\item message cost function $c$
%\end{itemize}
\begin{tabular}{@{}l@{\hspace{10pt}}ll}
	Given & \textbullet\ $N$ training listeners & $L = \{l_i\}_{i=0}^{N-1}\in \mathcal{O}\times\mathcal{I}\rightarrow \mathcal{A}  \text{ }(i=0, 1, \dots, N-1)$ sampled from $\mathcal{D}_{\text{listener}}$\\
	& \textbullet\  Language game environment & $E$: $\mathcal{O}\times\mathcal{A}\rightarrow \mathcal{O}$\\
	& \textbullet\  Speaker & $S$: $\mathcal{O}\times\mathcal{G}\rightarrow\mathcal{I}^+\times\mathcal{A}$\\
	& \textbullet\  Message cost function & $C$: $\mathcal{I} \rightarrow \mathbb{R}$\\
	& \textbullet\ Constants & cost coefficient $\kappa\in\mathbb{R}$, distribution coefficient $\sigma\in[0, 1]$,\\
	&&maximum number of interactions $K \in \mathbb{N}$
\end{tabular}

\setlength{\abovedisplayskip}{0pt}
\setlength{\belowdisplayskip}{0pt}
\setlength{\abovedisplayshortskip}{0pt}
\setlength{\belowdisplayshortskip}{0pt}
% \vspace{5pt}
While not converged: %Repeat until convergence:
\vspace{-5pt}
\begin{enumerate}
\item[1.] Define dataset $\mathcal{D}_{\theta_{\text{mind}}}(l_i) = \{(o_j, m_j, a_j)\}$ for each training listener $l_i$ and game. For a given game, the goal is $g$; the first observation is $o_1$; the message and action are\\[-5pt]
\begin{align}
     M, a_j^g &= S(o_j, g)\\
     \mathcal{Q}(M)\ &=\underset{m\in M}{\text{normalize}}(\mathcal{P}_{\text{ToM}}(a_j^g \mid o_j, m, \{(o_k, m_k, a_k)_{k=1}^{j-1}\}; \theta_{\text{mind}})\exp(-\kappa C(m))) \quad \label{eq:instruction_distribution}  \\
     m_j &\sim \sigma \mathcal{Q}(M) + (1-\sigma)\mathcal{U}(M) \quad a_j = l_i(o_j, m_j) \quad o_{j+1} = E(o_j, a_j)
\end{align}

where $a_j^g$ is the planned action of the speaker; $\underset{m\in M}{\text{normalize}}$ represents normalizing unnormalized probabilities.\\[-15pt]
\item[2.] Compute prediction loss \\
\vspace{-5pt}
\begin{align}
    \mathcal{L}^{\text{pred}}(\mathcal{D}_{\theta_{\text{mind}}}) = -\mathbb{E}_{i\sim \mathcal{U}([N]),k \sim \mathcal{U}([K]), \mathcal{D}_{\text{supp}} \sim \mathcal{U} (\mathcal{D}_{\theta_{\text{mind}}}^k(l_i)), (o, m, a) \sim \mathcal{U}(\mathcal{D}_{\theta_{\text{mind}}}(l_i))} \log \mathcal{P}_{\text{ToM}}(a \mid o, m, \mathcal{D}_{\text{supp}};\theta_{\text{mind}})
\end{align}

where $i$ is the index of the listener, $k$ is the size of the support set which are uniformly sampled from $\{0, 1, \dots, N-1\}$ and $\{0, 1, \dots, K-1\}$, the support set $\mathcal{D}_{\text{supp}}$ and target sample $(o, m, a)$ are sampled from $\mathcal{D}_{\theta_{\text{mind}}}$ uniformly.
\item[3.] Update the ToM parameters:  $ \theta_{\text{mind}} \leftarrow \argmin_{\theta} \mathcal{L}^{\text{pred}}(\mathcal{D}_{\theta_{\text{mind}}})
$
\end{enumerate}

\end{tcolorbox}
\end{small}
\renewcommand{\figurename}{Procedure}
\captionof{procedure}{General Theory-of-Mind (ToM) model training procedure.}
\label{proc:training_tom}
\end{figure*}

\subsection{Population}
\citet{rabinowitz2018machine} coined the notion of ``machine ToM'', which is a model for tracking agents' behaviors. To train and evaluate its ability to adapt to different agents, they create populations of subject agents by using different neural architectures and random seeds for parameter initialization. 

Our design of populations draws inspiration from their work. However, to make the similarities and differences between agents controllable, we consider a population as a distribution over parameters neural listeners with the same architecture which have been trained on different datasets. A neural listener $f_{\theta}^L: \mathcal{O}\times\mathcal{I}\rightarrow\mathcal{A}$ is a mapping from $\mathcal{O}$bservations and $\mathcal{I}$nstructions to $\mathcal{A}$ctions with parameter $\theta$ of the neural networks. The parameters trained on dataset $\mathcal{D}$ are:
\begin{equation}
    \theta_{\mathcal{D}} = \arg\min_{\theta} \mathcal{L}(f_L, \mathcal{D})
\end{equation}

% \gn{``species'' hasn't been defined yet, and \citet{rabinowitz2018machine} hasn't been discussed yet, so this is confusing. Maybe explain briefly what they do and that we draw inspiration from them, but we decide to do things differently because XXX.} 

In this way, the variation over listeners is mainly determined by the features of the dataset. By constructing datasets with different feature distributions, we control the listeners' language abilities.
% \gn{more meaningful compared to what?} 

% \gn{I've tried to modify this a bit. Could you take a look and see if it looks good?}
In the example of Fig.~\ref{fig:demo}, the observations in language games are the items in the visual field, and the agent may perform any number of actions (e.g. ``\texttt{\color{blue}go to \color{red}cabinet}'').
The speaker may provide natural language instructions, which can range from high-level (e.g.~``cool the winebottle'') to low-level  (e.g.~``get the wine bottle'', ``take it to the refrigerator'', ``put it in the refrigerator''). 
% \gn{you haven't described on a high level what the language game you're playing is, so ``color'' and ``shape'' are abrupt. If you want to use the running example from Figure 1 that's good, but it'd be nice to at least briefly note that you're considering that setting and try to link things like ``observations'' ``instructions'' and ``actions'' to that example.}
A listener that has never been trained on a particular variety of (usually high-level) instruction would have trouble performing the appropriate actions.
This leads to an exponential population of listeners that are trained on datasets containing, or not containing, particular relevant instructions.
Because of this, in order to effectively and concisely communicate, an effective speaker will have to judge the language abilities of its various partners in the population and adjust appropriately; we explain how we do so in the following section. 

% Besides this feature of the existence of ``cool \& place'' instructions, we consider the similar features of other instructions. The space of different listeners grows exponentially to the number of features. In this way, the population contains listeners with a wide range of language abilities.  

\subsection{Theory-of-mind Model}

% \gn{I don't think $\kappa$ is described in the main text of this section anywhere, but it should be, especially because it plays an important role in the experiments.}

Theory-of-mind, the ability to build a model of one's conversational partners, is deemed to be crucial in socio-pragmatics theory \cite{premack1978does,tomasello2018children}.
Drawing inspiration from this, we build a theory-of-mind model to learn to mimic the behavior of the listener within a short time window. 

\paragraph{Mental State} Modeling mental states is the central concept in building a theory-of-mind.
We define the mental state of the listener as the parameters of a neural model, the \emph{ToM model},
% \gn{I don't think you've defined the concept of a ``ToM model'' or how it may be parameterized yet, so this is a bit unclear.}
that produces the same output for the same inputs as the listener: $\forall x\in \Sigma^*, o \in \mathcal{O}, g_{\text{ToM}}(x, o;\theta_{\text{mind}}) \approx g(x, o;\theta)$. It should be noted that in the general case, particularly when different model architectures are used to represent the model itself and the ToM model, the mental state representations may not be unique or even exist.
In other words, for any model $\theta$ there may be more than one parameter setting $\theta_{\text{mind}}$ that satisfies this condition, or there may be no $\theta_{\text{mind}}$ that produces the exact same output.%\gn{reworded this sentence}

\paragraph{Building a Theory-of-mind}
Learning a ToM model is reduced to 
inferring the mental state of the listener. For a given listener $g$ with parameters $\theta$ and ToM model $g_{\text{ToM}}$, we seek a mental state representation $\theta_{\text{mind}}$. In practice, we use identical neural architectures for both the listener and ToM Model. However, inferring the exact mental state is infeasible within few interactions. Therefore, we estimate $g_{\text{ToM}}$ such that
\begin{equation}
    \theta_{\text{mind}} = \arg\min_{\theta'} \mathbb{E}_{o, m}\mathcal{L} (g_{\text{ToM}}(o, m;\theta'), g(o, m;\theta))
\end{equation}

It is straightforward to apply this definition of mental state in the psychological context for which it was originally proposed.
The mental state $\theta_{\text{mind}}$ is the representation of the listener's language abilities, which are not directly observable, and which are ultimately used for predicting the belief and behavior of the speaker \cite{premack1978does}. 
% \gn{This would be much easier to understand if you used the running example from alfworld, is there a reason why this is not feasible?}
For example, in our first set of experiments we focus on referential game where the speaker describes the target in order to let the listener pick it out from distractors. We construct a population in which neural listeners with LSTMs and word embeddings have different language comprehension abilities for different languages.
One of the possible representations controls the word embeddings in different languages: the mental state of a good language listener should have more meaningful word embeddings, while the one which cannot understand the language should have more random ones. Given that the speaker can acquire an accurate mental state for the listener, it can be used for predicting the probability of listener choosing the correct image when hearing descriptions in different languages. By choosing the one that yields the correct image with the highest probability, the speaker generates the descriptions which improve the referential game. On the other hand, high quality descriptions help the speaker better narrow down the language abilities of the listener. This is similar to the two-way interrelation between language and ToM in humans \cite{de2007interface}. 

Following this direction, we present a dynamic view of ToM by putting the observer inside the conversation, instead of the static view of \citet{rabinowitz2018machine}, which uses ToM for tracking the behavior of the agent without interfering in the games. Our training procedure is presented in Proc.~\ref{proc:training_tom}.  
% \gn{I'd put the discussion of the similarity to DAgger after you describe the algorithm first. If you were talking about a very widely used algorithm like REINFORCE I'd suggest mentioning it first, but DAgger is a little more esoteric, so people might feel they're obligated to read this paper and understand DAgger first unless we're careful. Also, in the main text it'd be a good idea to explain each of the equation steps in turn in the main text.} 
We aggregate a dataset $D_{\theta_{\text{mind}}}$ at each epoch, and update the parameters by optimizing the ToM model on the dataset. To aggregate the dataset for each training listener, we randomly sample from the posteriors of the ToM model and uniform distributions over the candidates, which keeps a certain degree of exploration, modulated by distribution coefficient $\sigma$ (through the paper, we use $\sigma = 0.5$). In practice, parameters are updated with stochastic gradient descent by sampling listeners and using the history of each listener at each time step as a support set for predicting the next actions of the listener. Following the literature on speech acts, e.g.~\citet{monroe2015learning}, we also add exponential cost penalty $\exp(-\kappa C(m))$ as a prior to penalize long instructions. (We have not explored the space of penalty functions in this paper, but the exponential function is widely used in the pragmatics literature, e.g. \cite{monroe2015learning}, \cite{morris2019pressure}.) In Fig.~\ref{fig:demo} (a\&b), although ``go to fridge'' yields the highest probability of gold action, no instruction is given in order to express the goal concisely. 

Similarly to the imitation learning algorithm DAgger \cite{ross2011reduction}, the dataset is collected using expert actions. However, there is a major difference between Proc. \ref{proc:training_tom} and DAgger --- we optimize the prediction of actions conditioned on the observations and instructions instead of the instruction probability directly.
% \gn{The following remainder of this subsection doesn't seem very important to the central idea of this paper, and can probably go in an appendix.}
The following theorem shows that our model will improve the instruction generation quality:

\begin{theorem}[informal] 
\label{th:1}
Given a small enough distribution coefficient $\sigma$ and good enough bounded candidate pools, the instruction distribution produced by the ToM model becomes optimal as prediction loss goes to zero.  
\end{theorem}

%\yb{I think adding a sentence here or there will help make it clearer (e.g. what are i, k, ... in the $s$ definition) -- maybe just spell out in a sentence what the expectations are over.}
%\hao{See Proc. \ref{proc:training_tom}}
% I still worry this is a little abrupt but hopefully the reader is paying attention

\paragraph{Discussion} The conditions of Theorem \ref{th:1} mean that the speaker model $S$ must be a well-trained model to produce good enough candidates pools. In practice, this condition is not hard to meet: 
% \gn{commented out}
% in our referential game experiments, the listeners can understand at least one language, and the speaker can produce descriptions in all 10 languages that the listeners can possibly understand; while 
for instance, in our language navigation experiment the listeners can at least understand the lowest-level instructions, and the speaker generates four levels of instructions by rule-based experts. Therefore, the practical implication of this theorem is helpful -- our method reduces to DAgger without expert instructions. Different from DAgger, our training method doesn't directly optimize the instruction distribution against expert's instructions, but optimizes the action prediction loss instead, which upper-bounds the instruction loss.

\subsection{Meta-learning ToM Model}
To acquire an estimate of the mental state from very few interactions, the ToM model needs to quickly learn from a small support set. In theory, any model for parameterizing $\mathcal{P}_{\text{ToM}}(a \mid o, m, \mathcal{D}_{\text{supp}};\theta_{\text{mind}})$ could work in our framework.
As a general method applicable to all ToM models, we apply model-agnostic meta-learning (MAML; \citet{finn2017model}), a method that explicitly trains models to be easily adaptable in few-shot settings. Given support dataset $\mathcal{D}_{\text{supp}}$, the inner loop updates parameters for $N_{\text{inner}}$ steps:
\begin{equation}
\begin{aligned}
    \theta^{(0)} &= \theta_{\text{mind}} \\
    \theta^{(i+1)} &= \theta^{(i)} - \eta \nabla_{\theta} \mathcal{L}(\mathcal{D}_{\text{supp}}; \theta)|_{\theta=\theta^{(i)}}
\end{aligned}
\end{equation}
where $\eta$ is the inner loop learning rate, which, in practice, is not share across modules following \citet{antoniou2018train}. The loss function on the support set is the negative log-likelihood of listener's action given observations $o$ and the instructions $m$
%\yb{split intwo two sentences for clarity}
\begin{equation}
    \mathcal{L}(\mathcal{D}_{\text{supp}}; \theta) = - \mathbb{E}_{(o, m, a)\sim\mathcal{D}_{\text{supp}}} \log p_{\theta}(a\mid o, m)
\end{equation}

After $N_{\text{inner}}$ steps, we get the prediction on the target observation $o$ and instruction $m$ 
\begin{equation}
    \mathcal{P}_{\text{ToM}}(a \mid o, m, \mathcal{D}_{\text{supp}};\theta_{\text{mind}}) = p_{\theta^{N_{\text{inner}}}}(a\mid o, m)
\end{equation}

Outer loop optimize $\theta_{\text{mind}}$ by mini-batch stochastic gradient descent
\begin{equation}
    \theta_{\text{mind}} \leftarrow \theta_{\text{mind}} - \eta_{\text{outer}} \nabla_{\theta} \mathcal{L}^{\text{pred}} (\mathcal{D}_{\theta_{\text{mind}}})
\end{equation}

The outer loop also runs for a given $N_{\text{outer}}$ epochs. 

\subsection{Deploying the ToM Model}
Similar to Proc.~\ref{proc:training_tom}, we evaluate ToM by using it to measure the probability of the gold action given the instructions. However, here we choose the best one instead of sampling from the posterior. Alg.~\ref{proc:test_tom} shows the evaluation procedure. 
\begin{algorithm} % enter the algorithm environment

\small % make it small
\caption{Evaluate ToM Model} % give the algorithm a caption
\label{proc:test_tom} % and a label for \ref{} commands later in the document
\begin{algorithmic}
 \REQUIRE Testing Listeners $L_{test}$, $E, S, C, \kappa, K$ as in Proc. \ref{proc:training_tom}
 \STATE $pt \leftarrow 0$
\FORALL{$l_i$ in $L_{test}$}
    \STATE $\mathcal{D}\leftarrow \emptyset$
    \STATE $o, a \leftarrow$ \textsc{Restart}
    \FOR{j in 1..K}
        \STATE{$o, g\leftarrow E(o, a)$}
        \IF{\textsc{Done}}
        \STATE $pt \leftarrow pt + \bm{1}[\textsc{Success}]$
        \STATE $o, g\leftarrow E(\textsc{Restart})$
        \ENDIF
        \STATE $M, a^g\leftarrow S(o, g)$
        \STATE $m \leftarrow \argmax_m \mathcal{P}_{\text{ToM}}(a^g \mid o, m, \mathcal{D}\}; \theta_{\text{mind}})e^{-\kappa C(m)}$
        \STATE $a \leftarrow l_i(o, m)$
        \STATE $\mathcal{D} \leftarrow \mathcal{D} \cup \{(o, m, a)\}$
    \ENDFOR
    
\ENDFOR
\STATE \textbf{Return} $pt$
\end{algorithmic}
\end{algorithm}

\subsection{Connection with Other Pragmatics Models}
It should be noted that using a listener model to help choose best utterance has been studied for almost a decade under the rational speech act model (RSA, \citet{frank2012predicting}; including recent more general models, e.g. \citet{Wang2020AMT}), a Bayesian framework that takes listener's choices in to account by 
\begin{equation}
\begin{aligned}
    P_{S^{n}} (m\mid a, o) &= \frac{P_{L^{n-1}}(a\mid m, o) P (m\mid o) }{\sum_{m'\in M}P_{L^{n-1}}(a\mid m', o) P(m'\mid o)}\\
    P_{L^{n}} (a\mid m, o) &= \frac{P_{S^{n-1}}(m\mid a, o) P (a\mid o) }{\sum_{a'\in A}P_{L^{n-1}}(m\mid a', o) P(a'\mid o)}
\end{aligned}
\end{equation}
where $S^n$ denotes the $n$-level speaker and $L^{n-1}$ denotes the $(n-1)$-level listener, $M, A, o$ denotes the space of instructions, actions, and the observation shared by the speaker and listener respectively, $P(m\mid o)$ and $P(a\mid o)$ are the priors over instructions and actions. The base speaker $S^0$ and listener $L^0$ are often parameterized using neural networks directly \cite{fried2018unified}.  

As a general framework for computational pragmatics, RSA models both language production and language comprehension in a recursive fashion, although the first and the second levels are predominantly used. In this paper, we focus on language production, while improving the listeners with more layers of reasoning is left for future work. 

However, the most notable difference between our model and neural RSAs is the notion of few-shot coordination. RSA base speaker and listener models are often fixed after training, making them unable to adapt to new partners during testing. While our model has a similar formulation (Eq. \ref{eq:instruction_distribution}) to the first level speaker of RSA, our ToM listener's action probability conditions on the listener's previous behavior. 

% \begin{figure*}[!h]
% \includegraphics[width=\textwidth]{referential-game-model.pdf}
% \caption{Referential game model}   
% \label{fig:referential_game_model} 
% \end{figure*}

\section{Multilingual Referential Games}

We test the ability of the proposed ToM model to perform few-shot language coordination in two settings: the running example of vision-language navigation, and also in a simpler setting of \emph{referential games}, which we discuss first in this section.
In a referential game, the speaker gives a description for the target image as its instruction, and the listener's action is to choose the target from distractors, after which the listener either wins the game and gets one point or loses it. 

Following \citet{lazaridou2016multi, lowe2019interaction}, we use 30k image-caption pairs from MSCOCO dataset \cite{lin2014microsoft}.
In each game, a target image sampled from the dataset uniformly, and nine distractors are sampled from 1,000 nearest images in terms of cosine similarity of outputs of second last layer of pretrained ResNet \cite{he2016deep}. 
In contrast to previous work, which mainly deals with a pair of one speaker and one listener, we are interested in learning with a population of listeners.
In order to achieve this, we propose a setting of \emph{multilingual} referential games, where each listener has the ability to understand different languages at different levels of ability.

\paragraph{Listener distribution}
We first translate MSCOCO captions into nine languages, German, Lithuanian, Chinese, Italian, French, Portuguese, Spanish, Japanese and Greek, from English, using Google Translate\footnote{\url{https://translate.google.com}}. For each listener, we sample a vocabulary distribution $v_1,v_2,\dots,v_{10}$ from 10-dimensional Dirichlet distribution $Dir(0.5, 0.5, \dots, 0.5)$. The listener's vocabulary is built up with 5,000 words, where for each language $i$ we select the most frequent $5,000*v_i$ words in MSCOCO captions in that language to be added to the listener's vocabulary. The reason behind this design is cognitively motivated; word frequency has high correlation with age of acquisition (AoA) of words \cite{juhasz2005age}.
The dataset used to train the listener is finally created by filtering out sentences with more than one word outside the vocabulary. Given target image $x^*$, instruction $m$, and distractors $x_{i}, i=1, 2, \dots, 9$, the listener computes
\begin{equation}
\label{eq:listener}
\begin{aligned}
    z_i &= \mathtt{ResNet}(x_i)\quad\text{for }i=1, 2, \dots, 9, *\\
    z &= \mathtt{LSTM}(m)\\
    \hat{y} &= \mathtt{softmax}(z^{\top} \{z_1, z_2,\dots,z_9, z^*\})
\end{aligned}
\end{equation}
The listener is trained to minimize the expected negative log-likelihood $-\log \hat{y}^*$ by stochastic gradient descent.

Following \citet{lowe2019interaction},
% \gn{when citations are used as part of the sentence use ``citet''} 
we train the listeners by randomly\footnote{We have also tried other schemes in their paper, but those do not yield significantly better performance.} interleaving between self-play (training with a companion speaker) and supervised training (with MSCOCO annotations or their translations). The companion speaker takes the representation of the target image as input: 
\begin{equation}
\label{eq:speaker}
\begin{aligned}
    z^* &= \mathtt{ResNet}(x^*)\\
    l &= \mathtt{teacher\text{-}forcing} (\mathtt{LSTM}(z^*), m) \\
    \hat{m} &= \mathtt{gumbel\text{-}softmax}(\mathtt{LSTM}(z^*))
\end{aligned}
\end{equation}
During supervised training, the model is trained to minimize the teacher-forcing NLL loss, while during self-play the sampled instruction is fed to the listener with Gumbel-softmax \cite{jang2016categorical}. This procedure produces 120 listeners, for which the average success rate with MSCOCO captions within the listener's vocabulary is 81.6\% and the average success rate with companion speakers is 83.3\%. These listeners are randomly divided into training, validation, and testing listeners (80/20/20). 

\paragraph{Speaker training} Using the setup in Eqs.~\ref{eq:listener} and \ref{eq:speaker}, we equip the speaker with a vocabulary of 20K words equally distributed in ten languages. We use the same data filtering method and training scheme as described above. To produce a pool of candidates in all languages, we add a language marker at the front of each training caption, so that the languages of instructions are controllable. Using beam search (size of 10), we generate five instructions per language (i.e.~$N_M\!\!=\!50$). The speaker achieves an 87\% success rate with the listeners used to train the speaker
%\gn{I don't think ``training listener'' has been explained?}
and a caption PPL of 23.7. 

\paragraph{ToM Model} The ToM models uses the same architecture as Eq.~\ref{eq:listener}. We use penalty $\kappa=0$.  In the referential game, the action space $\mathcal{A}=\{1, 2, \dots, 9, *\}$ and observation $o = (x_1, x_2, \dots, x_9, x^*)$, we have
\begin{equation}
    p_{\theta}(a\mid o, m) = \hat{y}_{a}.
\end{equation} 
The MAML hyper-parameters are $\eta=0.01, N_{\text{inner}} = 5, \eta_{\text{outer}} = 0.0001, N_{\text{outer}}=500$, and batch size is 2. 
\vspace{-2pt}
\paragraph{Evaluation} We evaluate the ToM-assisted speaker and other baselines with the same set of testing listeners. For each pair of speaker and listener, we calculate the average success rate of 500 $K=20$-game sessions. 
\begin{table}[!h]
\centering
\small
\begin{tabular}{lr}\\\toprule  
Model & Ave success\\\midrule
Gold-standard speaker &91.20\%\\ \hline\hline
Non-ToM speaker &37.38\%\\  \hline
RSA speaker & 42.83\%\\ \hline
ToM-assisted speaker &\textbf{58.19}\%\\  \bottomrule
\end{tabular}
\caption{Models and their respective referential game accuracy.}\label{wrap-tab:referential-game}
\end{table}

The gold-standard speaker denotes the success rate of using the testing listener in place of the ToM listener. The score of over $90\%$ indicates that the candidate pool is of high quality, so a speaker with a well-modeled ToM listener has ample room for achieving high accuracy. The non-ToM speaker uses the instruction with the highest probability in the speaker model; the RSA speaker uses the listener for training the speaker in place of the ToM listener. Our model achieves a significantly higher success rate, demonstrating that the ToM model could help produce better instructions for this referential game.

\begin{figure}
    \centering
    \includegraphics[width=0.7\linewidth]{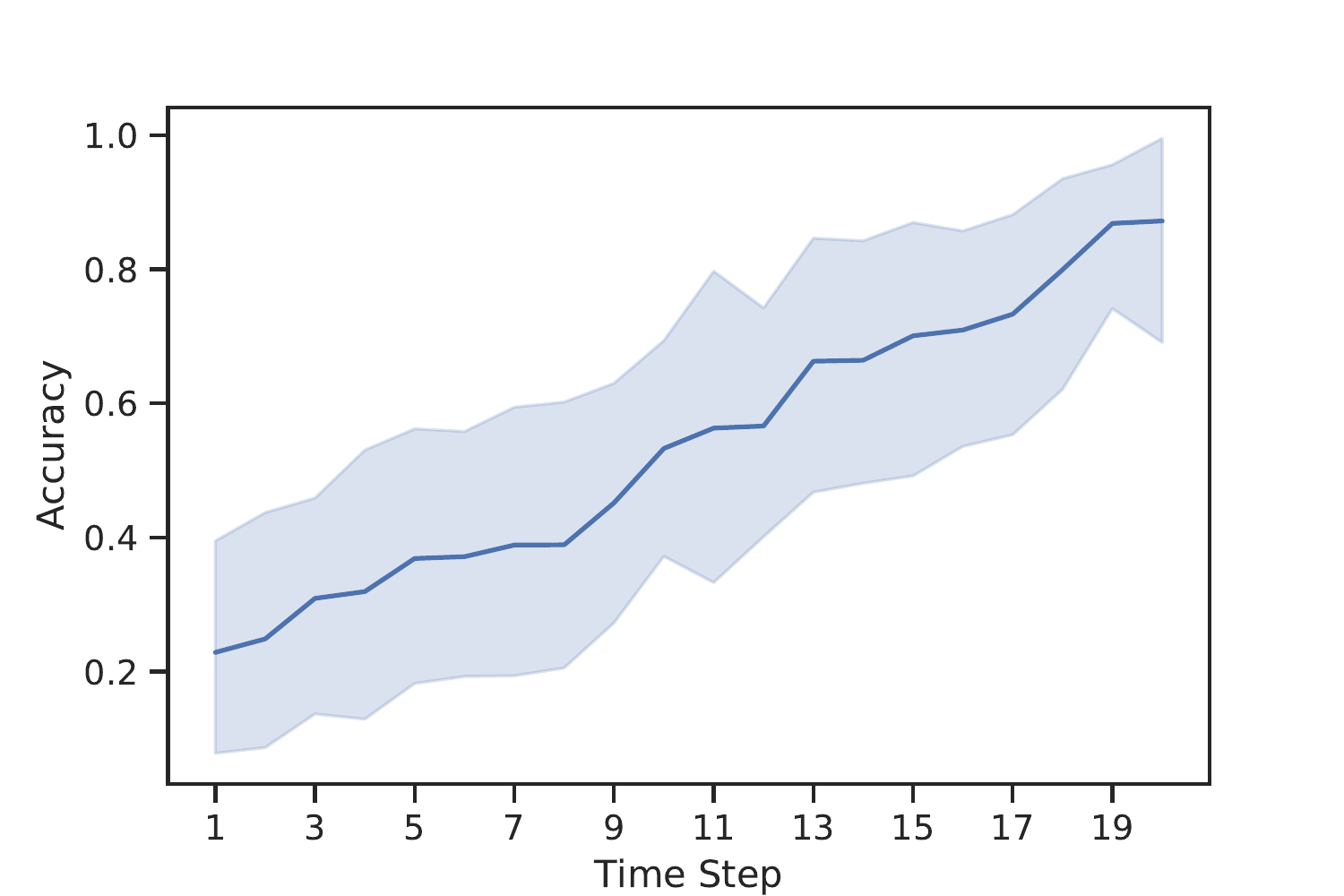}
    \caption{Average prediction accuracy of ToM model at each time step during evaluation. (95\% confidence interval)}
    \label{fig:acc}
\end{figure}
However, does ToM model truly learn to adapt to individual listeners? We compute the accuracy of predicting the listener's behavior during the same session. 
% \gn{The following would be really good to make into a figure. Right now my biggest impression about the paper is that it's a bit short on content, so having more figures/tables/interesting analysis would greatly strengthen the paper. This one seems like low-hanging fruit.}
Fig.~\ref{fig:acc} shows that the prediction accuracy of listener's actions is significantly improved within sessions, which shows ToM indeed learns to adapt to individual test listeners. 

\section{ALFWorld Task Navigation}
In the previous section, we have shown that ToM could help games with simple dynamics, and learn to adapt to listeners. This section will show its application in a more complex game: language navigation. 

\begin{figure}[!t]
    \centering
    \includegraphics[width=0.8\linewidth, trim=0 5 0 0, clip]{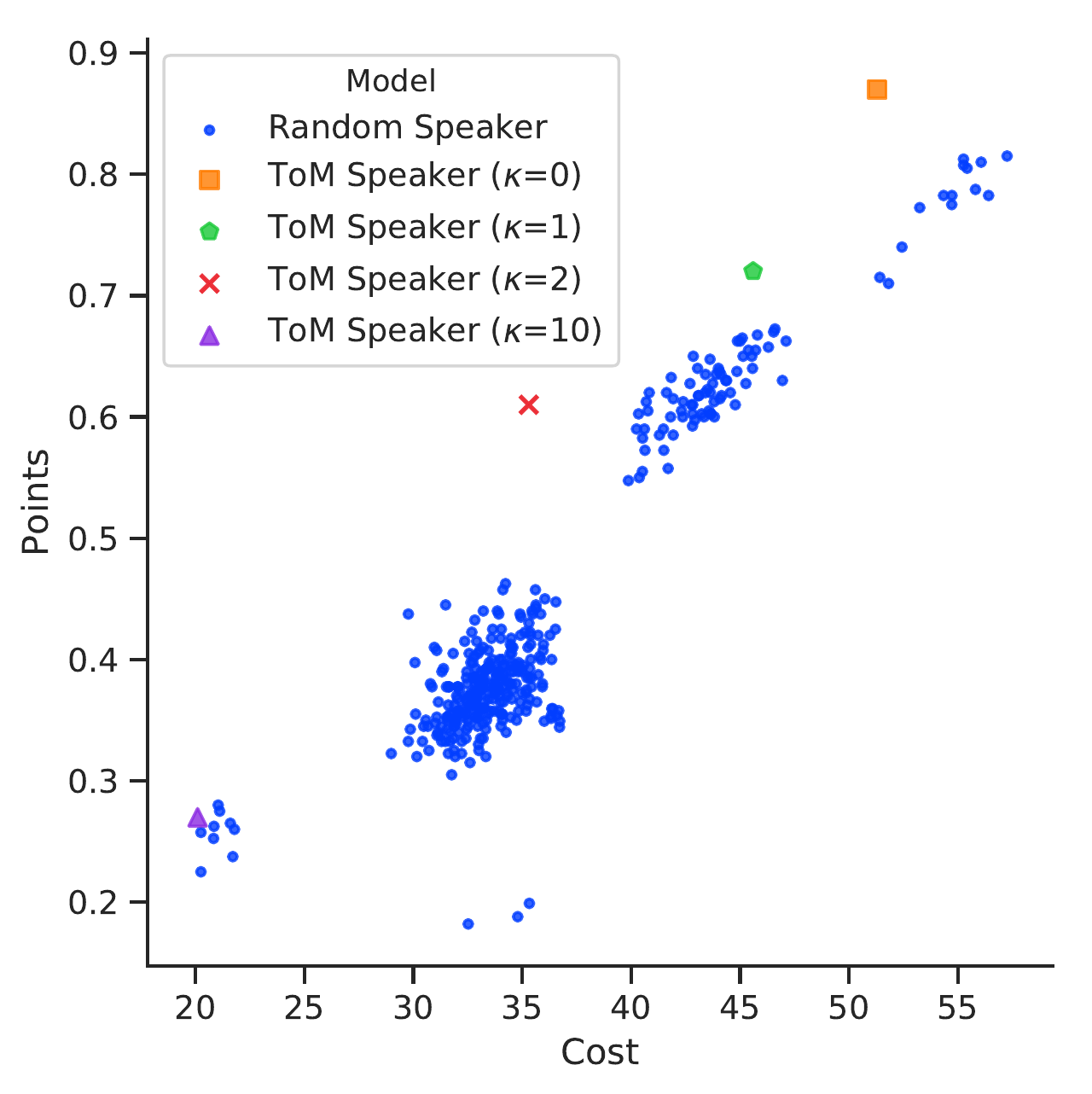}
    \caption{Experimental results for the language navigation setting with average instruction length on the horizontal-axis and game points on the vertical. Colors represent different models. 
    }
    \label{fig:language_navigation}
\end{figure}

We use the Alfworld \cite{shridhar2020alfworld} platform, which creates a natural language command action space upon the Alfred environment for vision-language navigation \cite{ALFRED20}. In each game, a household task is given by the environment. We generate expert trajectories for all tasks in Alfworld, and manually create four levels of instructions, from task-level to action-level, which are denoted as $\mathcal{I}_i, i=1, 2, 3, 4$. The differences between these four levels are the levels of abstraction. An action-level instruction corresponds to a single action in the Alfworld, while a task-level instruction corresponds to a whole trajectory which consists of more than eight commands. The two other levels are in between. The candidate pool at each time step consists of four instructions from each level ($N_M=4$). To create each listener, we draw an instruction distribution from $Dir(0.6, 0.4, 0.3, 0.2)$ for each of the six types of task in the environment. Listeners are of the same neural model as in \citet{shridhar2020alfworld}. While training the listeners, the instructions are randomly drawn from the instruction distributions according to the task type. This procedure produces 50 listeners. The average success rate of listeners is 83.6\%. These listeners are randomly divided into training, validation, and testing listeners (30/10/10). We define the cost as the average total length of the set of instructions, i.e.~repetitive instructions are only calculated once. The cost function is defined as $C(m) = 2^i \ \text{if } m \in \mathcal{I}_i$. Within one session, the maximum number of interactions between speaker and listener is $K=100$, and maximum number of interactions in a game is 20. Listeners' hyper-parameters are the same as the ones in \citet{shridhar2020alfworld}, while MAML hyper-parameters are the same as referential game. 

In Fig.~\ref{fig:language_navigation}, we compare ToM-assisted speakers and random speakers. We didn't compare with an RSA speaker, because differently from the multilingual referential games, the speaker is rule-based and no listener is used for training the speaker. A random speaker draws an instruction distribution from $Dir(0.7, 0.7, 0.7, 0.7)$, and sample instructions from the candidate pools using the instruction distribution. The ToM speaker with $\kappa=0$ predominantly uses action-level instructions, while the ToM speaker with $\kappa=10$ uses task-level instructions most of the time. Comparing ToM speakers and random speakers, we find that for $\kappa = \{0, 1, 2\}$, ToM speakers achieve higher game points and lower cost than random ones; for $\kappa=10$, the ToM speaker does not have significant improvement over a random one, since only the listeners that have been trained on sufficient action-level instructions can succeed.

\section{Related Work}

\subsection{Language Games}
Language games have been the proving ground for various linguistics theories since their conception by \citet{wittgenstein1953philosophical}. Recently, the most widely used language game is the referential game, in which the speaker observes the target and distractors and uses language to instruct the listener on how to pick out the target. 

\textbf{Emergent Communication} Without natural language annotations, this pressure for the speaker and listener enables language emergence. \citet{batali1998computational} first uses the same recurrent neural networks as the speaker and the listener to conduct emergent communication in referential game. Following this lead, \citet{lazaridou2016multi} study how emergent languages are grounded to the input images. \citet{cao2018emergent} studies multi-turn communication via negotiation. \citet{chaabouni2020compositionality,gupta-etal-2020-compositionality} study the compositionally and systematicity of emergent languages.

\textbf{Learning Common Languages} By using natural language annotations, agents learn a common language so that agents that are never trained together can be expected to communicate. \citet{lazaridou2016multi} studies using MSCOCO \cite{lin2014microsoft} annotations as gold labels for both speakers and listeners. \citet{lowe2019interaction} found that alternating between self-playing and supervised learning benefits communication performance. \citet{wang2016games} show that humans have the ability to adapt to the machine's capability in language games.  \citet{bullard2020exploring} found that when language follows Zipf law, zero-shot communication is possible. 

\textbf{Language Games with Community} \citet{tieleman2019shaping}  learns representations by training with a community of encoders and decoders. The difference between our work and theirs is that our MAML listener learns to adapt to different listeners in the population in a few games. The performance of their model should be equivalent to our model's result at time step 1. \citet{lowe2019learning} considers the adaptation problem, which is definitely relevant. However, adapting their model to our settings is non-trivial, which is beyond the scope of this paper. 

This referential game setting used in most previous work can be seen as a special case of our few-shot coordination formulation, where the number of games is one and the partners are mostly the same ones as in the training phase. These two differences prevent the previous models from learning to adapt due to the lack of pressure to do so. 
\subsection{Machine Theory of Mind}
Computational and neural models of theory-of-mind have been studied for decades \cite{siegal2002neural,rescorla2015computational}. \citet{rabinowitz2018machine} are the first to present a successful modeling of the mental state of various species of agent. While we also train a ToM model with meta learning, we put the ToM model into use. The predictions provided by the ToM model serve as reranker in the speaker's model. The variety of models is also more diverse than the species used in this paper. Importantly, \citet{nematzadeh2018evaluating,le2019revisiting} find that neural models for question answering fail to keep track of inconsistent states of the world.  \citet{moreno2021neural}  extends machine ToM to neural recursive belief states. We expect improvement over our current model by modeling higher-order recursive belief, which is left for future work. \citet{yuan2020emergence} explicitly trains belief state prediction with supervised learning, while our model's belief state is latent. 

\subsection{Similar Topics in Reinforcement Learning}

\textbf{Model-based Reinforcement Learning} Model-based reinforcement learning focuses on building a model of the environment to improve data efficiency \cite{kaelbling1996reinforcement}, which could be applied to the zero-shot coordination problem by treating the listener as a part of the environment. Recently, neural networks have been used widely for model-based RL \cite{gal2016improving, depeweg2016learning, nagabandi2018neural,chua2018deep,janner2019trust}. We should point out that despite their similarities to our model, %in the following sections, 
we focus on modeling different and unseen agents in the population within a few interactions.

\textbf{Alternatives to Self-play} 
Zero-shot coordination has attracted much attention recently. \citet{hu2020other} propose other-play which maximizes the expected reward working with random partners. In contrast to their approach, we explicitly model the partner's ToM and focus on language coordination, which required more complicated modeling than the environments in their experiments. 
% \subsection{Knowledge Tracing}
% HZ: knowledge tracing literature
% e.g. deep knowledge tracing

\section{Implications and Future Work}
We have introduced few-shot language coordination task and proposed ToM model for tracking the listener's mental state. Different from previous work using single-round games and self-play training, we consider more general multi-round games and playing with novel listeners. ToM model shows its ability to adapt to novel listeners and assist speakers in choosing the best instructions in both multilingual referential games and the language navigation task. We attribute the success of ToM model to modeling socio-pragmatics process in an explicit way. Many interesting questions about modeling ToM in few-shot language coordination remain open. The most immediate is how to model ToM from the listener's perspective, leading to a dialog agent that can acquire new knowledge through conversation. One step further, similar to RSA, ToM can also be modeled in a recursive manner, which may further improve language games. 

\section*{Acknowledgements}
This work was supported by the DARPA GAILA project (award HR00111990063). The views and conclusions contained in this document are those of the authors and should
not be interpreted as representing the official policies, either expressed or implied, of the U.S. Government. The U.S. Government is authorized to reproduce and distribute reprints for Government purposes notwithstanding any copyright notation here on. 
\bibliography{example_paper}
\bibliographystyle{icml2021}
\clearpage
\appendix
\section{Formal Version of Theorem \ref{th:1}}
\begin{theorem}
In one epoch of Proc.~\ref{proc:training_tom}, if the ToM model is $\epsilon$-optimal, i.e.
$$\mathcal{L}^{\text{pred}} = \mathbb{E}_{s,m} KL[\mathcal{P}_{\text{ToM}}(a\mid m, s; \theta)\| P_{l_i}(a\mid o, m)] < \epsilon$$
where states $s = \langle i, k, \mathcal{D}_{\text{supp}}, o, m, g\rangle$ and instructions $m$ are sampled as Proc. \ref{proc:training_tom},
and for almost all states $s$ speaker gives a $\delta$-optimal instruction candidates pool $M$, i.e.
$$\sum_{m\in M} \mathcal{P}_{\text{ToM}}(a^g\mid m, s; \theta) \geq \delta$$
then expected KL-divergence 
\begin{equation}
    \mathbb{E}_s KL[\mathcal{Q}_{\text{ToM}}(m\mid s)\|\mathcal{Q}(m\mid s; \theta)]
\end{equation}
between the instruction distribution calculated from ToM model
\begin{equation}
\label{eq:th1_result}
    \mathcal{Q}_{\text{ToM}}(m\mid s; \theta) \triangleq
    \frac{\mathcal{P}_{\text{ToM}}(a^g\mid m, s; \theta)}{\sum_{m'\in M}\mathcal{P}_{\text{ToM}}(a^g\mid m', s; \theta)}
\end{equation}
and the target instruction distribution
\begin{equation}
    \mathcal{Q}(m \mid s) \triangleq
    \frac{P_{l_i}(a^g\mid o, m)}{\sum_{m'\in M}P_{l_i}(a^g\mid o, m')}
\end{equation}
upper-bounded by
\begin{equation}
      \frac{N_M\sqrt{\frac{\epsilon}{2(1-\sigma)}} + W_0(\epsilon)}{\delta}
\end{equation}
where $N_M$ is the size of largest pool of instruction candidates produced by the speaker, and $W_0$ is the principle branch of Lambert's W function.  
\end{theorem}

% \begin{theorem}
% \label{th:1}
% In one epoch of Proc.~\ref{proc:training_tom}, if the ToM model is $\epsilon$-optimal, i.e.
% $$\mathcal{L}^{\text{pred}} = \mathbb{E}_{s,m} KL[P_{l_i}(a\mid o, m)\| \mathcal{P}_{\text{ToM}}(a\mid s; \theta)] < \epsilon$$
% where states $s = \{i, k, \mathcal{D}_{\text{supp}}, o, m, g\}$ and instructions $m$ are sampled as Proc. \ref{proc:training_tom},
% and for almost all states $s$ speaker gives a $\delta$-optimal instruction candidates pool $M$, i.e.
% $$\sum_{m'\in M} \mathcal{P}_{\text{ToM}}(a^g\mid m, s; \theta) \geq \delta$$
% then expected KL-divergence 
% \begin{equation}
%     \mathbb{E}_s KL[\mathcal{Q}_{\text{ToM}}(m\mid s; \theta)\| \mathcal{Q}(m\mid s)]
% \end{equation}
% between the instruction distribution calculated from ToM model
% \begin{equation}
% \label{eq:th1_result}
%     \mathcal{Q}_{\text{ToM}}(m\mid s; \theta) \triangleq
%     \frac{\mathcal{P}_{\text{ToM}}(a^g\mid m, s; \theta)}{\sum_{m'\in M}\mathcal{P}_{\text{ToM}}(a^g\mid m', s; \theta)}
% \end{equation}
% and the target instruction distribution
% \begin{equation}
%     \mathcal{Q}(m \mid s) \triangleq
%     \frac{P_{l_i}(a^g\mid o, m)}{\sum_{m'\in M}P_{l_i}(a^g\mid o, m')}
% \end{equation}
% upper-bounded by
% \begin{equation}
%      \frac{\frac{1}{2}N_M^2\epsilon}{\delta^2 (1-\sigma)^2}
% \end{equation}
% where $N_M$ is the size of largest pool of instruction candidates produced by the speaker.  
% \end{theorem}
\begin{proof}
Applying Pinsker inequality, 
\begin{equation}
\begin{aligned}
    \mathcal{L}^{\text{pred}} &= \mathbb{E}_{s,m} KL[\mathcal{P}_{\text{ToM}}(a\mid m, s; \theta)\| P_{l_i}(a\mid o, m)]\\ 
    &\geq \mathbb{E}_{s,m} 2TV(P_{l_i}(a\mid o, m), \mathcal{P}_{\text{ToM}}(a\mid m, s; \theta))^2 \\
    &= \mathbb{E}_{s,m} 2\sup_a |P_{l_i}(a\mid o, m) - \mathcal{P}_{\text{ToM}}(a\mid m, s; \theta)|^2 \\
    &\geq \mathbb{E}_{s,m} 2|P_{l_i}(a^g\mid o, m) - \mathcal{P}_{\text{ToM}}(a^g\mid m, s; \theta)|^2\\
    &\geq \mathbb{E}_{s,m} 2|\Delta(s, m))|^2 \\
    &\geq 2(1-\sigma)\mathbb{E}_{s}\mathbb{E}_{m\sim \mathcal{U}(M)} |\Delta(s, m))|^2 \\
    &\geq 2(1-\sigma)(\mathbb{E}_{s}\mathbb{E}_{m\sim \mathcal{U}(M)} |\Delta(s, m))|)^2
\end{aligned}
\end{equation}
where $\Delta(s, m) = P_{l_i}(a^g\mid o, m) - \mathcal{P}_{\text{ToM}}(a^g\mid m, s; \theta)$.

By processing the target expectation
\begin{equation}
\begin{aligned}
    &\mathbb{E}_s KL[\mathcal{Q}_{\text{ToM}}(m\mid s; \theta)\| \mathcal{Q}(m\mid s)]\\
    =&\mathbb{E}_s \log \frac{\sum_{m'\in M}P_{l_i}(a^g\mid o, m')}{\sum_{m'\in M}\mathcal{P}_{\text{ToM}}(a^g\mid m', s; \theta)}
    \\
    & + \mathbb{E}_s \frac{\sum_{m\in M}\log\frac{\mathcal{P}_{\text{ToM}}(a^g\mid m, s; \theta)}{P_{l_i}(a^g\mid o, m)}\mathcal{P}_{\text{ToM}}(a^g\mid m, s; \theta)}{\sum_{m'\in M}\mathcal{P}_{\text{ToM}}(a^g\mid m', s; \theta)} \\
    \leq & \frac{N_M}{\delta} \mathbb{E}_s\mathbb{E}_m\Delta(s, m') \\
    &+ \mathbb{E}_s \frac{\sum_{m\in M}W_0(KL[\mathcal{P}_{\text{ToM}}(a\mid m, s; \theta)\| P_{l_i}(a\mid o, m)])}{\delta} \\
    =& \frac{N_M\sqrt{\frac{\epsilon}{2(1-\sigma)}} + W_0(\epsilon)}{\delta}
\end{aligned}
\end{equation}
\end{proof}
\section{Training Time and space}
All of our models can be trained on a 32 Gb V100. A model (speaker, listener, or ToM model) for referential game trains for about 20 hours, while a model (speaker, listener, or ToM model) for language navigation trains for 72 about hours. Tab. \ref{wrap-tab:referential-game} and Fig. \ref{fig:language_navigation} reports the average of three runs, Fig. \ref{fig:acc} reports data from 20 testing listeners. 

\begin{table}
\centering
\footnotesize
\vspace{-10pt}
\begin{tabular}{@{}l@{\hspace{-30pt}}r@{}}\\\toprule  
Model & Ave success (\%)\\\midrule
Gold-standard speaker &91.20\\ 
\midrule
Non-ToM speaker &37.38\\  
RSA w/ single listener  & 39.32 \\ 
RSA speaker & 42.83\\ 
Finetuned RSA & 44.30 \\ 
\midrule
ToM. speaker (\begin{scriptsize}large $h\!=\!768$\end{scriptsize}) &55.28\\ 
ToM. speaker (\begin{scriptsize}small $h\!=\!256$\end{scriptsize}) &56.75\\ 
ToM. speaker ($N_{\text{inner}}\!=\!1$)  &56.10\\
ToM. speaker ($N_{\text{inner}}\!=\!10$) &58.25\\ 
ToM. speaker &\textbf{58.19}\\  \bottomrule
\end{tabular}
\caption{The influence of various hyperparameters}
\label{fig:appendix_hyperparameters}
\end{table}

\section{Hyper-parameter Tuning}
We only tuned the inner and outer learning rates of MAML among $1e^i, i= -1, -2, -3, -4, -5$. A few influential hyperparameters are shown in Tab. \ref{fig:appendix_hyperparameters}. Other parameters are all kept same as previous work: \citet{lowe2019interaction} for referential game, and \citet{shridhar2020alfworld} for language navigation.
\end{document}